\documentclass[conference]{IEEEtran}
\IEEEoverridecommandlockouts

\usepackage{cite}
\usepackage{amsmath,amssymb,amsfonts}
\usepackage{algorithmic}
\usepackage{graphicx}
\usepackage{textcomp}
\usepackage{xcolor}
\usepackage{bm}
\usepackage{mathtools}
\usepackage{subcaption}
\usepackage{tikz}
\usepackage{tikz-3dplot}
\usepackage{float}
\usepackage{geometry}
\usepackage{soul}
\usepackage{algorithm}
\usepackage{microtype}
\usepackage{amsthm}  
\usepackage{booktabs}
\usepackage[colorlinks=true, linkcolor=blue, citecolor=blue, urlcolor=blue]{hyperref}
\newtheorem{theorem}{Theorem}
\newtheorem{remark}{Remark}
\newtheorem{assumption}{Assumption}

\def\BibTeX{{\rm B\kern-.05em{\sc i\kern-.025em b}\kern-.08em
    T\kern-.1667em\lower.7ex\hbox{E}\kern-.125emX}}

\begin{document}

\title{Systematic Analysis of Coupling Effects on Closed-Loop and Open-Loop Performance in Aerial Continuum Manipulators
\\
{\footnotesize}

\thanks{This work was supported by the National Research Council Canada, grant number AI4L-128-1, and the Natural Sciences and Engineering Research Council of Canada, grant number 2023-05542. 
Niloufar Amiri, Shayan Sepahvand, and Farrokh Janabi-Sharifi are with the Department of Mechanical, Industrial, and Mechatronics Engineering, Toronto Metropolitan University, Toronto, ON M5B 2K3, Canada (emails: \texttt{\{niloufar.amiri, shayan.sepahvand, fsharifi\}@torontomu.ca}). 
Iraj Mantegh is with the Aerospace Research Center, National Research Council of Canada (NRC), Montreal, QC H3T 2B2, Canada (email: \texttt{iraj.mantegh@nrc-cnrc.gc.ca}).}%
}

\author{
Niloufar Amiri\IEEEauthorrefmark{1}, 
Shayan Sepahvand\IEEEauthorrefmark{1}, 
Iraj Mantegh\IEEEauthorrefmark{2},
Farrokh Janabi-Sharifi\IEEEauthorrefmark{1}
}

\maketitle
\begin{abstract}
This paper investigates two distinct approaches to the dynamic modeling of aerial continuum manipulators (ACMs): the decoupled and the coupled formulations. Both open-loop and closed-loop behaviors of a representative ACM are analyzed. The primary objective is to determine the conditions under which the decoupled model attains accuracy comparable to the coupled model while offering reduced computational cost under identical numerical conditions. The system dynamics are first derived using the Euler–Lagrange method under the piecewise constant curvature (PCC) assumption, with explicit treatment of the near-zero curvature singularity. A decoupled model is then obtained by neglecting the coupling terms in the ACM dynamics, enabling systematic evaluation of open-loop responses under diverse actuation profiles and external wrenches. To extend the analysis to closed-loop performance, a novel dynamics-based proportional-derivative sliding mode image-based visual servoing (DPD-SM-IBVS) controller is developed for regulating image feature errors in the presence of a moving target. The controller is implemented with both coupled and decoupled models, allowing a direct comparison of their effectiveness. The open-loop simulations reveal pronounced discrepancies between the two modeling approaches, particularly under varying torque inputs and continuum arm parameters. Conversely, the closed-loop experiments demonstrate that the decoupled model achieves tracking accuracy on par with the coupled model (within subpixel error) while incurring lower computational cost.
\end{abstract}

\section{Introduction}
In recent years, uncrewed aerial vehicles (UAVs) have found wide-ranging applications in monitoring, delivery, and inspection \cite{mohsan2022towards}. However, performing specific tasks often requires UAVs to actively interact with objects and the environment, which necessitates the integration of additional components such as cables and robotic arms. While this integration increases the complexity of coupled modeling and control, it is essential for enabling active UAV intervention in grasping and aerial physical interactions \cite{khamseh2018aerial}.

Multi-rotor UAVs are agile aerial platforms well suited for maneuvering in confined indoor and outdoor environments. However, their payload capacity is limited compared to fixed-wing UAVs, as lift is generated solely by the thrust of rotary wings. This limitation highlights the requirement for careful consideration of any additional payload to be mounted on them \cite{kumar2024comprehensive}. As an initial approach, rigid robotic arms have been integrated into UAVs for aerial manipulation \cite{ollero2021past}. More recently, aerial continuum manipulators (ACMs) have been introduced to integrate the advantages of continuum robots (CRs) \cite{peng2025dexterous, ubellacker2024high}.

CRs are modular or continuous, lightweight robotic systems constructed from hyper-elastic materials, which enable substantial deformation along their entire structure and support smooth transitions between kinetic and potential energy \cite{russo2023continuum}. Their low stiffness and high damping properties make them appropriate for safe interaction with surfaces and objects, minimizing harsh impacts. Another key feature of CRs is their high adaptability to the surrounding environment, due to their virtually infinite degrees of freedom, which allow them to navigate effectively in unstructured and confined spaces \cite{wooten2022environmental}.

Despite these advantages, the dynamic behavior of hyper-elastic soft materials is highly nonlinear. The large deformation of CRs arises from distributed actuation along the entire body, rather than being concentrated at discrete joints. Consequently, dynamic modeling of CRs is computationally expensive, often infeasible, and largely incompatible with well-established controllers in robotics. This challenge becomes even more pronounced when CRs are coupled with UAVs \cite{boyer2020dynamics}.

The primary focus of this work is analytical dynamic modeling of the coupled ACM. While finite element methods \cite{ferrentino2023finite} and data-driven techniques \cite{liu2025data} provide alternative modeling approaches, analytical models \cite{Jalali2022} offer advantages such as efficiency and insight into the coupling effects between the UAV and its arm. Coupled and decoupled models of the integrated system each bring certain advantages. The coupled dynamics of ACMs provide an accurate representation of all system states and capture the effects of CR maneuvers on the UAV, such as shifts in the center of mass and changes in the moment of inertia. These effects are more pronounced in certain scenarios and may cause UAV instability if neglected. Although coupled dynamics offer a more realistic model, they have several drawbacks. A coupled rigid-soft analytical model can result in stiff partial differential equations, as the numerical solver must use small time steps to capture infinitesimal CR deformations, even though such small steps are unnecessary for UAV dynamics \cite{renda2020geometric}. Alternatively, a decoupled approach allows for efficient and simpler control system development.

\section{Related Works}
In aerial manipulation, tendon-driven CRs have been more widely adopted than those actuated by pressurized fluids \cite{sokolov2023design} or concentric tubes \cite{mitros2022theoretical} with external actuation units. This preference is attributed to several advantages, including ease of assembly, a more compact form factor, and higher payload capacity. In tendon-driven CRs, actuation is achieved through distributed tension applied by tendons routed along the robot’s body. Although the continuous backbone of the CR can theoretically undergo linear and shear local deformations, bending remains the most practically controllable mode in existing prototypes \cite{uthayasooriyan2024tendon}.

CR dynamics have traditionally been studied using classical theories such as Timoshenko \cite{haghshenas2023timoshenko} and Euler–Bernoulli beams \cite{rao2021using}, which rely on specific deformation modes and simplified assumptions, making them valid only within limited regions. A more accurate representation of slender CRs is provided by Cosserat rod theory \cite{till2019real, janabi2021cosserat}, which accounts for the general deformation of the robot. This theory leads to boundary-conditioned nonlinear partial differential equations (PDEs) that require numerical methods with high computational costs. Several variants of Cosserat rod theory have been developed, featuring different discretization schemes in time and space while maintaining comparable accuracy \cite{tummers2023cosserat}. Decoupled from the floating base, the CR has been modeled using Cosserat rod theory and applied to aerial manipulation in several works \cite{hashemi2023robust, samadikhoshkho2020modeling}. In addition to models that consider the continuum mechanics of CRs, other existing approximation models are more simplified, such as the redundant pseudo-rigid body \cite{huang20193d}, refined lumped mass \cite{zhang2024adaptive}, and center-of-gravity approaches \cite{godage2015accurate}. These models require case-by-case identification techniques and are less accurate for dynamic tasks, as they neglect the internal properties of the material used.

In the literature, due to the recency of ACMs, their dynamic modeling has not been fully investigated. Existing models are predominantly based on simplifying kinematic assumptions such as piecewise constant curvature (PCC) and piecewise constant strain (PCS) \cite{peng2025dexterous, peng2023aecom}, although variable-strain approaches have been introduced more recently \cite{AMIRI2025}. Leveraging the PCC assumption, a comprehensive dynamic model of an ACM has also been formulated using the Euler–Lagrange method, making it compatible with traditional model-based controllers \cite{samadikhoshkho2021coupled}. Owing to the inherent compatibility of the framework, the model has been further extended by researchers to accommodate a coupled dual-arm ACM \cite{ghorbani2023dual}. Furthermore, a coupled model of an ACM has been derived using an augmented rigid-body model adapted from rigid manipulators with the same kinematic assumption in \cite{szasz2022modeling}. In this model, the elastic behavior of the CR is represented through virtual springs and dampers. Even though the coupled model reflects ACM dynamics comprehensively by capturing the CR’s inertia on the UAV, it significantly increases computational complexity. Given that the CR’s nonlinear behavior already raises the cost compared to rigid aerial manipulators, the added burden from coupling terms necessitates detailed investigation.

In this study, we propose a robust image-based visual servoing (IBVS) controller, motivated by two main considerations. The primary challenge is that pose estimation of the ACM in GPS-denied or cluttered environments is highly challenging, if not infeasible. A further requirement is that the controller must tolerate various sources of uncertainty, including hysteresis, unmodeled dynamics, and tendon slacking \cite{SAMADI2022Reduced}. Beyond these considerations, we further investigate how the choice between a decoupled and a coupled model impacts the performance of perception-guided control in both the image and task spaces.

The contributions of this work are threefold. First, we conduct a comparative study of the factors that lead to discrepancies in the task-space open-loop response between the decoupled and coupled approaches for a single-section ACM. This analysis takes into account different actuation profiles, external wrenches applied at the CR tip, variations in UAV mass, and changes in CR length and stiffness. Second, we develop a novel dynamics-based proportional-derivative sliding mode IBVS (DPD-SM-IBVS) controller for target tracking with an eye-in-hand camera configuration. The Lyapunov stability framework is employed to guarantee the global ultimate boundedness of the image feature error. Finally, to highlight the differences in closed-loop error dynamics between the two approaches, we evaluate a visual tracking task in which the target follows distinct trajectories, including both linear and curved paths.

\section{Problem Formulation}

Prior to describing the problem, the following notations are adopted. The pre-superscript is utilized to show the frame in which the vector is defined, e.g. $\prescript{B}{}{\bm{P}}$ shows the vector is defined in the frame $B$, while $\prescript{B}{}{\bm{R}}_{A}$ indicates the rotation from frame $A$ to frame $B$.

\subsection{Single-Section CR Differential Kinematics}

In this work, we consider a single-section ACM as the first step in getting insight into coupling effects in ACMs. The configuration variables of the CR are illustrated in  \linebreak Fig. \ref{fig:CR_Frames}. In this figure, $r$ is the radius of the curvature, which is the reciprocal of curvature, $\kappa = \frac{1}{r}$, and $\theta_a = \frac{s}{r} = \kappa s$ where $s$ is the spatial variable along the  CR length. The forward kinematics of a single-section CR based on PCC assumption is obtained as follows:
\vspace{-0.1cm}
\begin{equation}
     \prescript{B}{}{\bm{P}}_{s} = \prescript{B}{}{\bm{P}}_{o}+{\mathcal{\bm F}}(s,\kappa,\psi_a),
\label{eq:eq_s_position}
\end{equation}

\vspace{-0.1cm}

\noindent where  $\prescript{B}{}{\bm{P}_s} = [x_{s}, y_{s}, z_{s}]^\top$ is the position of the spatial element in the base frame with $0 \leq s \leq l_a$ and $l_a \in \mathbb{R}^+$  denoting the arm length, $\prescript{B}{}{\bm{P}_{o}} = [x_{o}, y_{o}, z_{o}]^\top$ is the offset from the base frame, $\psi_a$ is the orientation, and ${\mathcal{\bm F}}(.)$ is a nonlinear mapping from the configuration space to the operational space, expressed as \cite{Webster2010}


\begin{equation}
{\mathcal{\bm F}}(s,\kappa,\psi_a) = 
\begin{bmatrix}
\frac{1}{\kappa}(1 - \cos(s\kappa)) \cos(\psi_a) \\
\frac{1}{\kappa}(1 - \cos(s\kappa)) \sin(\psi_a) \\
\frac{1}{\kappa} \sin(s\kappa)
\label{eq:nonlinear_mapping}
\end{bmatrix}.
\end{equation}

The position at any point $s$ along the CR is obtained using \eqref{eq:eq_s_position}. Taking the time derivative of \eqref{eq:eq_s_position} yields:

\begin{figure}[t]
    \centering
    \includegraphics[scale=0.25, trim={7cm 1cm 5cm 1cm}, clip]{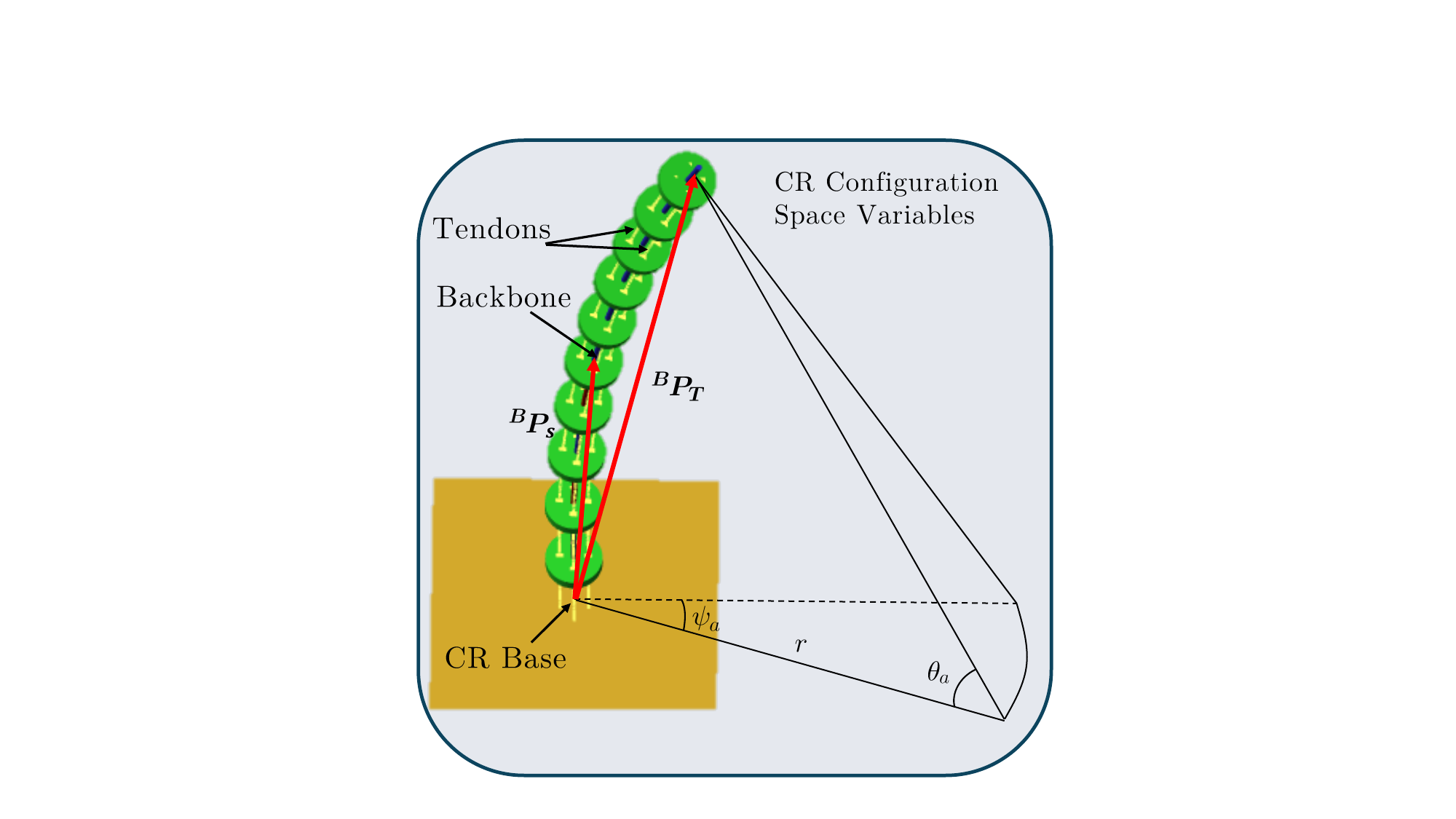}
    \caption{The configuration space variables of the CR under the PCC assumption.}
    \label{fig:CR_Frames}
\end{figure}

\begin{figure}[t]
    \centering
    \includegraphics[scale=0.25]{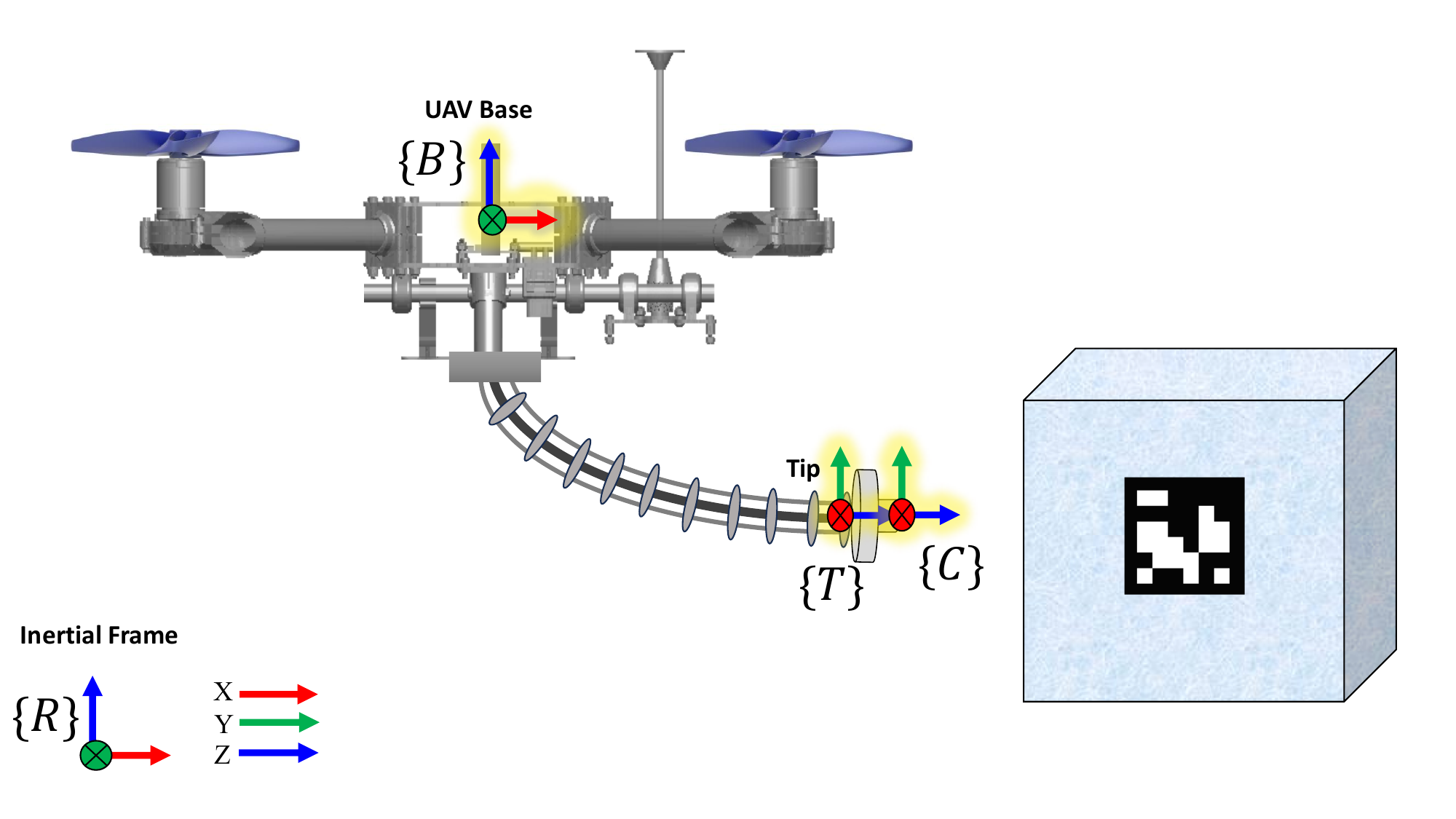}
    \caption{The frame attachment of the ACM.}
    \label{fig:ACMS_Frames}
\end{figure}
\vspace{-0.1cm}

\begin{equation}
    \partial \frac{\prescript{B}{}{\bm{P}}_{s}}{\partial t} = \frac{\partial \prescript{B}{}{\bm{P}}_{s}}{\partial \kappa}\dot \kappa + \frac{\partial \prescript{B}{}{\bm{P}}_{s}}{\partial \psi_a}\dot \psi_.
    \label{eq:partial_der}
\end{equation}

\noindent Arranging \eqref{eq:partial_der} in a matrix representation will yield the translational part of the geometric Jacobian, \linebreak$
\bm{J}_{pa} = [\frac{\partial \prescript{B}{}{\bm{P}}_{s}}{\partial \kappa} , \frac{\partial \prescript{B}{}{\bm{P}}_{s}}{\partial\psi_a}]\in \mathbb{R}^{3\times2},$ given in the Appendix A. Taking the time derivative of the rotation matrix describing the orientation of the frame attached to the spatial element $s$, denoted by $\prescript{B}{}{\bm{R}}_{s} \in SO(3)$, yields:

\begin{equation}
    \prescript{B}{}{\bm{\dot{R}}}_{s} = \frac{\partial \prescript{B}{}{\bm{R}}_{s}}{\partial \kappa}\dot {\kappa} + \frac{\partial \prescript{B}{}{\bm{R}}_{s}}{\partial \psi_a} \dot{\psi}_a\in \mathfrak{so}(3).
    \label{eq:partial_der_R}
\end{equation}

\noindent Therein, $SO(3)$ and $\mathfrak{so}(3)$ represent the special orthogonal Lie group and Lie algebra, respectively. Using \eqref{eq:partial_der_R}, the rotational part of the Jacobian will be
\vspace{-0.1cm}

\begin{equation}
\bm{J}_{oa} = 
\begin{bmatrix}
-s \sin(\psi_a) & 0 \\
s \cos(\psi_a) & 0 \\
0 & 1
\label{eq:rotation_jacob}
\end{bmatrix}
\in \mathbb{R}^{3\times2}.
\end{equation}

\begin{remark}
To avoid the representation singularity in the Frenet–Serret \textup{\cite{Chirikjian1994}} direct kinematics at $\kappa=0$, a curvature threshold $\kappa_s \in \mathbb{R}^+ $is introduced. Specifically,
\begin{equation}
\kappa = 
\begin{cases}
\kappa_s & \text{if } |\kappa|\leq \kappa_s, \\
\kappa & \text{otherwise}.
\end{cases}
\end{equation}
This regularization ensures numerical stability of the direct and differential kinematics computations.
\end{remark}

\subsection{ACM Differential Kinematics}

Here, the equations of differential kinematics are presented by consideration of the kinematic coupling of the CR and the floating base. Recall the CR joint variables in \eqref{eq:eq_s_position}, which can be appended to the UAV states to construct the following generalized joint-space vector
\vspace{-0.1cm}

\begin{equation}
    \bm{q} = 
    \begin{bmatrix}
        x & y & z & \varphi & \theta & \psi & \kappa & \psi_a\\
    \end{bmatrix}^\top  \in \mathbb{R}^{8} ,
\end{equation}
where $x$, $y$, and $z$ are the position and $\varphi$ (roll), $\theta$ (pitch), and $\psi$ (yaw) are the ZYX Euler angles describing the pose of the base (aerial platform). The time derivative of the position of the spatial element $s$ expressed in the inertial frame ${R}$ referring to Fig. \ref{fig:ACMS_Frames} can be derived as
$
\bm{J}_{p} = 
\begin{bmatrix}
    \bm{J}_{p1} & \bm{J}_{p2} & \bm{J}_{p3}\\
\end{bmatrix}
$
with
$
    \bm{J}_{p1} = 
\bm{I}_3
$,
$
    \bm{J}_{p2} = -[\prescript{R}{}{\bm{R}}_{B} \prescript{B}{}{\bm{P}}_{s}]_{\times} \bm{T}_e^{-1}
$, and
$
    \bm{J}_{p3} = \prescript{R}{}{\bm{R}}_{B} \bm{J}_{pa}
$.
The rotational part of Jacobian is expressed as
$
\bm{J}_{o} = 
\begin{bmatrix}
    \bm{J}_{o1} & \bm{J}_{o2} & \bm{J}_{o3}\\
\end{bmatrix} 
$
with
$
    \bm{J}_{o1} = 
\bm{0}_3
$,
$
    \bm{J}_{o2} = \bm{T}_e^{-1}
$, and
$
    \bm{J}_{p3} = \prescript{R}{}{\bm{R}}_{B} \bm{J}_{oa}
$,
 where $\prescript{R}{}{\bm{R}}_{B}$ is the rotation matrix presenting the orientation of the UAV with respect to the inertial frame. Additionally, $\bm{T}_e(\Phi)$, with $\Phi = [\varphi, \theta, \psi]^\top$, is the mapping from the rate of the Euler angles to angular velocities according to the ZYX convention. The final Jacobian matrix can be presented as 
$
\bm{J} = 
\begin{bmatrix}
    \bm{J}_{p}^\top & \bm{J}_{o}^\top\\
\end{bmatrix}^\top  \in \mathbb{R}^{6\times8} 
$. 

\subsection{ACM Dynamics}

The Euler--Lagrange method is used to derive the dynamics of the ACM. 
The total kinetic energy $K$ and potential energy $U$ of the system are expressed as  
\begin{equation}
K = K_u + K_b + K_d,
\label{eq:kintetic}
\end{equation}
\begin{equation}
U = U_u + U_b + U_d,
\label{eq:potential}
\end{equation}
\noindent where $K_u$ and $U_u$ represent the kinetic and potential energy of the UAV, 
$K_b$ and $U_b$ correspond to those of the CR backbone, 
and $K_d$ and $U_d$ denote the energies associated with the spacer disks used for tendon routing. One may take advantage of the works conducted in \cite{samadikhoshkho2021coupled} and the differential kinematics \eqref{eq:partial_der}-\eqref{eq:rotation_jacob} to obtain the following dynamics in the standard Lagrangian form
\vspace{-0.1cm}

\begin{equation}
    \bm{M}\ddot{\bm{q}} + \bm{C}\dot{\bm{q}}  + \bm{G} = \bm{\tau} - \bm{J}_t^\top \bm{F}_e,
\label{eq:ACMS_Dynamics}
\end{equation}

\noindent where $\bm{M}$, $\bm{C}$, and $\bm{G}$ denote the generalized mass, Coriolis, and gravity matrices, respectively. 
The generalized control torque vector of the fully actuated UAV and the CR, and the point external wrench, exerted on the tip, are denoted by $\bm{\tau}$ and $\bm{F}_e$, respectively. The decoupled model is obtained by setting the off-diagonal elements of the dynamic matrices in \eqref{eq:ACMS_Dynamics} to zero. $\bm{J}_t$ is Jacobian expressed in the tip frame and computed by setting $s=l_a$. Two representative off-diagonal terms of the mass matrix $\bm{M}$ are provided in Appendix~B to demonstrate the highly nonlinear and complex dependence of the coupling terms on the system parameters and states, derived via extensive symbolic computations.

Overall, the role of some parameters is more transparent, such as the quadratic scaling with the CR radius $r_a$ and the linear scaling with the CR density $\rho$. In contrast, the dependence on the curvature $\kappa$ is less intuitive. Benefiting from asymptotic simplification, near the backbone base ($s\to 0$), the coupling terms become negligible, with both $M_{17}$ and $M_{18}$ vanishing at least quadratically in $s$ due to $\sin(\kappa s/l)$ and $\cos(\kappa s/l)-1$ expansions. Moreover, under a small-curvature approximation ($\kappa\to 0$), the apparent $1/\kappa^2$ factor in $M_{17}$ is canceled by the $\kappa^2$ scaling of the associated trigonometric combination, yielding a leading-order term that is independent of $\kappa$ and primarily scaled by the inertial factor $\rho r_a^2$. Conversely, $M_{18}$ retains an explicit curvature dependence. Since $\cos(\kappa s/l)-1=O(\kappa^2)$, the prefactor $1/\kappa$ results in $M_{18}=O(\kappa)$, and thus $M_{18}\to 0$ as $\kappa\to 0$. This confirms that curvature effects on the coupling dynamics are entry-dependent and cannot be generalized across all off-diagonal mass-matrix terms.

\begin{assumption}
We assume, similar to \textup{\cite{Ghosal_2024}}, that the contributions of the disks to the potential energy and kinetic energy are negligible in \eqref{eq:kintetic} and \eqref{eq:potential} in order to reduce the computational burden and model complexity. While \textup{\cite{samadikhoshkho2021coupled}} accounts for disk effects, their rigid attachment to the backbone only amplifies the magnitude of the coupling terms and is consistent with the observed behavior.

\end{assumption}

\subsection{DPD-SM-IBVS Control Design}

The mapping from the image dynamics to the camera velocity expressed in the camera coordinate frame is determined using the interaction matrix (image Jacobian). Let's consider the corner pixel coordinates  as \linebreak $\bm{s} = \left(u_1, v_1, ..., u_N, v_N\right)^T \in \mathbb{R}^{2N} $, and $\bm{s}_d$ denotes the desired image feature vector. Let the image feature error be defined as follows: 
\vspace{-0.1cm}

\begin{equation}
\bm{e} = \bm{s}_d - \bm{s}.
\label{feature_error}
\end{equation}

\noindent Differentiating \eqref{feature_error} with respect to the time yields:
\vspace{-0.1cm}

\begin{equation}
    \dot{\bm{e}} = -\bm{L}({\bm{v}_c}-\bm v_t),
    \label{feature_error_rate}
\end{equation}

\noindent thereby ${\bm{L}} = \left(\bm{L}_1, \bm{L}_2, ..., \bm{L}_n \right)^T \in \mathbb{R}^{2N \times 6}$ is the interaction matrix, $\bm v_c$ and $\bm v_t$ denote the camera and target velocity in local camera frame, respectively.

\begin{assumption}
We assume that the perspective projection model of the camera incurs negligible radial and tangential distortion so that \eqref{feature_error_rate} is satisfied. 
\label{as:assumption_distortion}
\end{assumption}

\begin{assumption}
The norm of desired target velocity \eqref{feature_error} and its first and second time derivatives are bounded, e.g. $||\bm v_t||<L_1$ $||\dot {\bm v}_t||<L_2$, and $||\ddot {\bm v}_t||<L_3$. This assumption was utilized in works such as \textup{\cite{Li2021}}.
\label{as:assumption_derivatives}
\end{assumption}

The normalized coordinates of $i^{th}$ feature \linebreak $\bm{s}_{ni} = \left(x_{ni}, y_{ni} \right)^T$ is derived as follows: 

\begin{equation}
    x_{ni} = (u_i - c_x)/f_x,\hspace{0.5cm}  y_{ni} = (v_i - c_y)/f_y.
    \label{coordinates}
\end{equation}

\noindent As the above values are normally reported in pixels, one can neglect the value of effective pixel size. Thus,

\begin{equation}
{\bm{L}}_i = 
\begin{bmatrix}
- \frac{1}{Z_i} & 0 & \frac{x_{ni}}{Z_i} & x_{ni} y_{ni} & -(1 + x_{ni}^2) & y_{ni} \\
0 & -\frac{1}{Z_i} & \frac{y_{ni}}{Z_i} & (1 + y_{ni}^2) & -x_{ni} y_{ni} & -x_{ni}
\end{bmatrix}.
\label{eq:image_jacobian}
\end{equation}

\noindent The parameters, $f_x$, $f_y$, $c_x$, and $c_y$ represent the focal length and the principal point coordinates obtained from the calibration process \cite{Amiri2024}. 

\begin{theorem}
\label{thm:GAS_nominal}
Consider the system dynamics \eqref{eq:ACMS_Dynamics} undergoing the joint-space control command \eqref{eq:PD_Joint} where $\bm{J}_L=\bm L \bm T_a^{-1}\bm J_t$, and the dynamics-based
proportional-derivative sliding mode image-based visual servoing (DPD-SM-IBVS) control law \eqref{eq:img_controller}

\begin{equation}
    \bm{\tau} = \bm{G} + \bm{J}_L^\top \big( \bm{K}_p \bm{e}_a - \bm{K}_d \bm{J}_L \dot{\bm{q}} \big)\in \mathbb{R}^{8},
    \label{eq:PD_Joint}
\end{equation}
\begin{equation}
    \bm{e}_a = \bm{C}_d \dot{\bm e} + \bm{C}_p \bm e + \bm{C}_s\tanh\!\left(\tfrac{\bm e}{\sigma}\right)  \in \mathbb{R}^{6},
\label{eq:img_controller}
\end{equation}
with
\begin{equation}
    \bm{T}_a = 
    \begin{bmatrix}
        \bm{T}_e(\Phi) & \bm 0_{3 \times 3} \\
        \bm 0_{3 \times 3} & \bm{T}_e(\Phi) 
    \end{bmatrix}  \in \mathbb{R}^{6\times6} .
\end{equation}

\noindent In case $\bm K_p, \bm K_d, \bm{C}_p, \bm{C}_d, \bm{C}_s\in \mathbb{R}^{6\times6}$ are symmetric positive definite gain matrices and $\sigma \in \mathbb{R}^{+} $, 
then the image feature error $\bm e$ is globally ultimately bounded.
\end{theorem}

\begin{proof}

To prove the theorem, we take the following Lyapunov candidate function into consideration: 
\begin{equation}
    V = \tfrac12\dot{\bm q}^\top \bm M \dot{\bm q} + \tfrac12 \bm e_a^\top \bm K_p \bm e_a.
    \label{eq:Lyap_def}
\end{equation}

\noindent Substituting the control laws \eqref{eq:img_controller}--\eqref{eq:PD_Joint} structure into the time derivative of \eqref{eq:Lyap_def} yields 

\begin{equation}
    \dot V \le -\tfrac12\,\bm v_c^\top \bm K_d \bm v_c - \lambda \|\bm e_a\|^2 + \delta,
    \label{eq:Vdot_bound}
\end{equation}
with $\bm v_c=\bm J_L \dot{\bm q}$ for some $\lambda>0$ and 
$
    \delta \;=\; 
    \frac{\|\bm K_p\|^2}{2\lambda}\,\|r(t)\|^2
    \;,
    \label{eq:delta_exact}
$
where
$
    r(t) \;=\; \bm C_d \ddot{\bm v}_t 
    \;+\; \bm C_p \dot{\bm v}_t 
    \;+\; \frac{\bm C_s}{\sigma}\,
        D_{\tanh}\!\left(\tfrac{\bm e}{\sigma}\right)\dot{\bm v}_t,
    \label{eq:r_term}
$
and $D_{\tanh}(\tfrac{\bm e}{\sigma})=\mathrm{diag}\!\big(\operatorname{sech}^2(e_i/\sigma)\big)$. Considering Assumption \ref{as:assumption_derivatives}, $\delta$ is upper bounded by 
$
\delta \le \frac{\|\bm K_p\|^2}{2\lambda}\;R_{\max}^2,
$
where
\[
R_{\max}=\|C_d\|\,L_3+\|C_p\|\,L_2+\frac{\|C_s\|}{\sigma}\,L_2,
\]

\noindent hence \eqref{eq:Vdot_bound} will be:
\vspace{-0.2 cm}
\begin{equation}
    \dot V \le -\tfrac12\,\bm v_c^\top \bm K_d \bm v_c - \lambda \|\bm e_a\|^2 + \frac{\|\bm K_p\|^2}{2\lambda}\;R_{\max}^2.
    \label{eq:Vdot_bound_complete}
\end{equation}

 Inequality~\eqref{eq:Vdot_bound_complete} shows that $V$ decreases outside a compact set. 
Hence, $(\bm v_c,\bm e_a)$ are uniformly ultimately bounded and converge to the set
\[
\Big\{(\bm v_c,\bm e_a)\;\big|\;
\tfrac{1}{2}\,\bm v_c^\top \bm K_d \bm v_c 
+ \lambda \|\bm e_a\|^2 
\;\le\; \tfrac{\|\bm K_p\|^2}{2\lambda}\,R_{\max}^2
\Big\}.
\]
Since $\bm e_a$ includes the linear term $\bm C_p \bm e$ together with bounded nonlinear and derivative terms, boundedness of $\bm e_a$ and $\bm v_c$ implies boundedness of $\bm e$. 
Therefore, the image feature error $\bm e$ is globally ultimately bounded.
\end{proof}


\section{Results}
In this section, we present the simulation results of the ACM performance under various open-loop and closed-loop scenarios. Across all simulations, identical numerical conditions are maintained, and thus any observed discrepancies do not arise from numerical solver or step size. Initially, the open-loop responses are analyzed for different input profiles, including impulse and chirp excitations, external wrench, and free-fall conditions, for both coupled and decoupled models. Next, the models are compared with respect to key parameters such as CR length, radius, stiffness, platform mass, initial CR bending, and UAV orientation. Eventually, a comprehensive analysis of discrepancies in ground target tracking performance across different trajectories is conducted to assess the impact of neglecting the coupling terms.

\subsection{Open-Loop Response Analysis}
The open-loop responses of the models are analyzed in four tests over a simulation time of $1 \ \mathrm{[s]}$, with an identical initial state vector  
$\bm{q} = \begin{bmatrix}
    \bm{0}_{1\times2} & 5\ \mathrm{[m]} & \bm{0}_{1\times3} & 0.1\ \mathrm{[m^{-1}]} & 0
\end{bmatrix}^\top.$ 

In Test~A, the wrench profile
$
\bm{\tau} =
\begin{bmatrix}
100 \ \mathrm{[N]} &
\bm{0}_{1\times7}\\
\end{bmatrix}^\top
$
is applied to the ACM. This represents a strong actuation of the UAV in the inertial $x$-direction with a constant force.  

Test~B investigates the response when the actuation input is set to zero and the ACM is affected only by gravity (free-fall). In Test~C, the generalized wrench vector is given by  

\begin{align}
\bm{\tau} &=
\begin{bmatrix}
\bm{0}_{6\times1} \\
0.1 \sin\!\left(20 t + \tfrac{\pi}{4}\right)\,\mathrm{[Nm]} \\
0.1 \cos\!\left(20 t\right)\,\mathrm{[Nm]}
\end{bmatrix}.
\end{align}

\noindent which represents a sinusoidal excitation on the CR, as opposed to Test~A where only the UAV is actuated.  

Test~D examines the effect of applying the following external wrench exerted to the CR tip, expressed in the inertial frame:  
\begin{align}
\bm{F}_e =
\begin{bmatrix}
10 \,\sin\!\left( 2\pi \Big(f_0 t + \tfrac{f_1 - f_0}{2T} t^2 \Big) + \tfrac{\pi}{4} \right)\ \mathrm{[N]} \\
10 \,\sin\!\left( 2\pi \Big(f_0 t + \tfrac{f_1 - f_0}{2T} t^2 \Big) \right)\ \mathrm{[N]} \\
25\ \mathrm{[N]} \\
\bm{0}_{3\times1}
\end{bmatrix}.
\end{align}

\noindent with $f_0 = 1\ \mathrm{[Hz]}$, $f_1 = 1\ \mathrm{[Hz]}$, and $T=10 \ \mathrm{[s]}$.

The metric considered for the task-space comparison across various scenarios is the normalized root mean square error (NRMSE), defined as
\begin{equation*}
\mathrm{NRMSE}_{\bm{W}} = \sqrt{ \frac{\sum_{i=1}^{m} \sum_{k=1}^{N} \left( \bm{W}_i^{\mathrm{c}}[k] - \bm{W}_i^{\mathrm{d}}[k] \right)^2}{\sum_{i=1}^{m} \sum_{k=1}^{N} \left( \bm{W}_i^{\mathrm{max}} - \bm{W}_i^{\mathrm{min}} \right)^2} },
\end{equation*}
\noindent where the superscripts c and d denote the coupled and
decoupled values of the metric $\bm W \in \mathbb{R}^{m\times N}$, respectively.
The translational and rotational errors of the end-effector are
treated as separate comparison metrics and are reported in
Table \ref{tab:openloop_comparison} as $\mathrm{NRMSE}_T$ and $\mathrm{NRMSE}_R$, respectively.

\begin{table}[t]
\centering
\caption{Comparison of coupled and decoupled model responses across open-loop tests.}
\label{tab:openloop_comparison}
\begin{tabular}{lcccccc}
\toprule
\textbf{Metric} & \textbf{Test A} & \textbf{Test B} & \textbf{Test C} & \textbf{Test D}\\
\midrule
NRMSE\(_{T}\) [m] & $0.0124$ & $0.2378$ & $\bm{0.3576}$ & $0.2464$  \\
NRMSE\(_{R}\) [rad] & $\bm {1.0404}$ & $0.2751$ & $0.5369$ & $0.3002$  \\
\bottomrule
\end{tabular}%
\end{table}

Overall, the discrepancy between the decoupled and coupled models is significant in certain cases and can result in a maximum cumulative error of approximately $0.36\ \mathrm{[m]}$ for the end-effector position and $1.04\ \mathrm{[rad]}$ for the end-effector rotation, as shown in Table \ref{tab:openloop_comparison}. The most pronounced discrepancy in end-effector position between the models occurs in Test C, whereas Test A exhibits only a negligible effect. This suggests that neglecting the coupling terms has minimal impact on the end-effector position when only the UAV is actuated. However, the coupling effects become non-negligible in the presence of a generalized external force acting on the UAV. In contrast, the rotational deviation is primarily influenced by the generalized actuation wrench of the UAV.

\subsection{Parameter Sensitivity Analysis}
In addition to the discrepancies observed in the open-loop responses of the decoupled and coupled models under different actuation profiles and external wrenches, an important consideration is their sensitivity to physical parameters. The significance of the coupling terms is evident in certain prototypes, particularly when the inertia of the CR is substantial and the overall ACM center of mass undergoes notable variations.      

The results of the parameter sensitivity analysis are illustrated in Fig.~\ref{fig:open_loop_params}. 
The figure presents the 2D tip trajectory, projected onto the xy plane of the inertial frame, corresponding to the following piecewise-defined wrench vector with an excitation frequency of $f = 10\ \mathrm{[Hz]}$ and amplitude $a_m=100 \ \mathrm{[N]}$ over a $5\ \mathrm{[s]}$ simulation:
\begin{align}
\bm{\tau} =
\begin{cases}
\begin{bmatrix}
a_m \sin(f t) &
-a_m \cos(f t) &
\bm{0}_{6 \times 1}
\end{bmatrix}^\top & t \leq 2.5 \ \mathrm{[s]}, \\[10pt]
\begin{bmatrix}
a_m \cos(f t) &
a_m \sin(f t) &
\bm{0}_{6 \times 1}
\end{bmatrix}^\top & t > 2.5 \ \mathrm{[s]}.
\end{cases}
\end{align} 

The first parameter investigated is the backbone radius. Three values of $1 \ \mathrm{[mm]}$, $5 \ \mathrm{[mm]}$, and $10 \ \mathrm{[mm]}$ are analyzed, as shown in Fig.~\ref{fig:open_loop_params}(a). Increasing the backbone radius raises the total inertia of the system, resulting in reduced displacement. In this scenario, the CR is not actuated and therefore behaves as an additional payload. As a result, for the same input wrench, a larger CR experiences less excitation, and the overall variation in the center of mass and moment of inertia of the system is smaller compared to the other cases. This explains why the coupled and decoupled responses are similar for this parameter value.

Fig.~\ref{fig:open_loop_params}(b) shows the 2D tip position for UAV masses ranging from 1 to 3 kg. Similar to the first scenario, increased inertia results in reduced displacement. The results indicate that as the masses of the UAV and the robotic CR become more comparable, the coupling terms become more significant, leading to greater discrepancies in the system that diverge over time.

With the emergence of ACM with long arms for delivery and manipulation \cite{peng2025dexterous}, another important physical parameter to investigate is the CR length. Fig.~\ref{fig:open_loop_params}(c) presents the open-loop task-space Cartesian tip position for different CR lengths. Compared to variations in CR radius, changes in total inertia are less pronounced due to the slender geometry of the CR. Longer CRs tend to reduce the discrepancies between the coupled and decoupled models, although the response becomes more oscillatory. As shown in the graphs corresponding to a CR length of $1.5\ \mathrm{[m]}$, neglecting the coupling terms prevents these oscillations from being captured accurately.

The effect of the backbone material is examined in Fig.~\ref{fig:open_loop_params}(d) for three different Young’s moduli corresponding to steel ($E = 207\ \mathrm{[GPa]}$), titanium alloys ($E = 120\ \mathrm{[GPa]}$), and Nitinol ($E = 40\ \mathrm{[GPa]}$). The stiffer arms exhibit only negligible differences in the 2D tip position compared to those with greater compliance. However, the influence of this parameter is less significant than that of the other parameters, particularly when there is no contact or external wrench applied to the end-effector.

The impact of the initial bending is investigated in Test E. This is illustrated in Fig.~\ref{fig:open_loop_params}(e), where the initial bending $\kappa_0$ is set to three distinct values. The results indicate that the larger the initial bending, the larger the deviation between the coupled and decoupled models in the task-space.

Finally, as a representative study, we also examined the effect of the UAV initial orientation on the response. Here, the UAV roll angle is tested against different initial values for $\phi$, as seen in Fig.~\ref{fig:open_loop_params}(f). The largest discrepancy between the two approaches is visible for the larger initial UAV roll angles.

\begin{figure}[t]
    \centering
    \begin{subfigure}[t]{0.48\linewidth}
        \centering
        \includegraphics[scale=0.28]{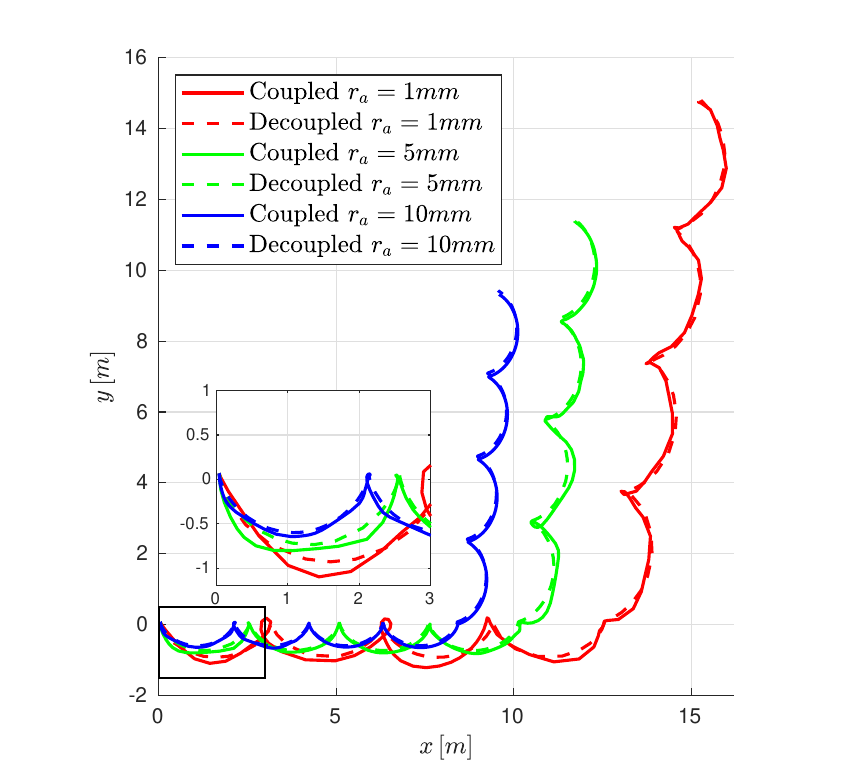}
        \caption{Different arm radii}
        \label{fig:arm_radius}
    \end{subfigure}
    \hfill
    \begin{subfigure}[t]{0.48\linewidth}
        \centering
        \includegraphics[scale=0.28]{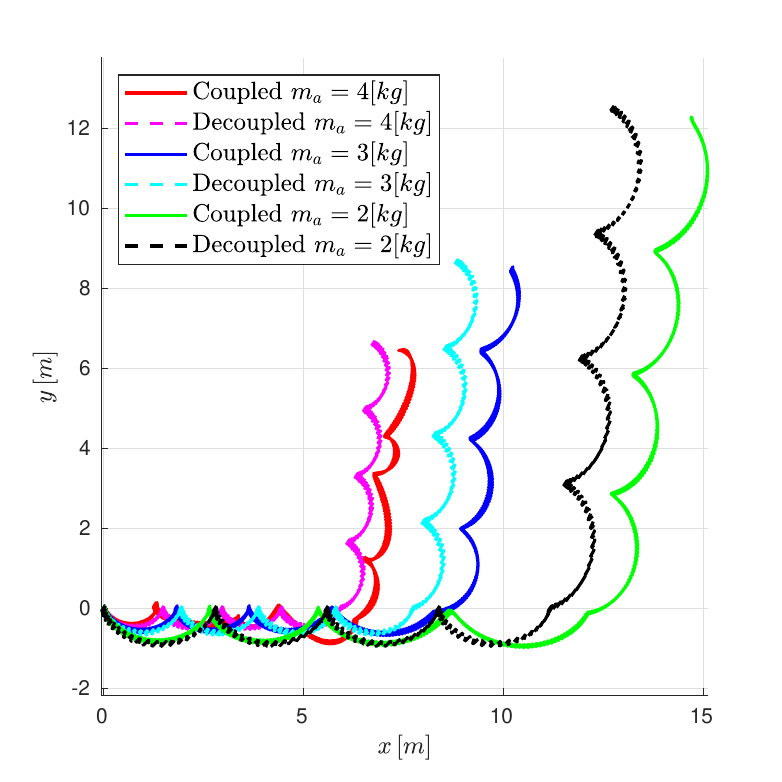}
        \caption{Different UAV masses}
        \label{fig:uav_mass}
    \end{subfigure}

    \vspace{0.5em}

    \begin{subfigure}[t]{0.48\linewidth}
        \centering
        \includegraphics[scale=0.28]{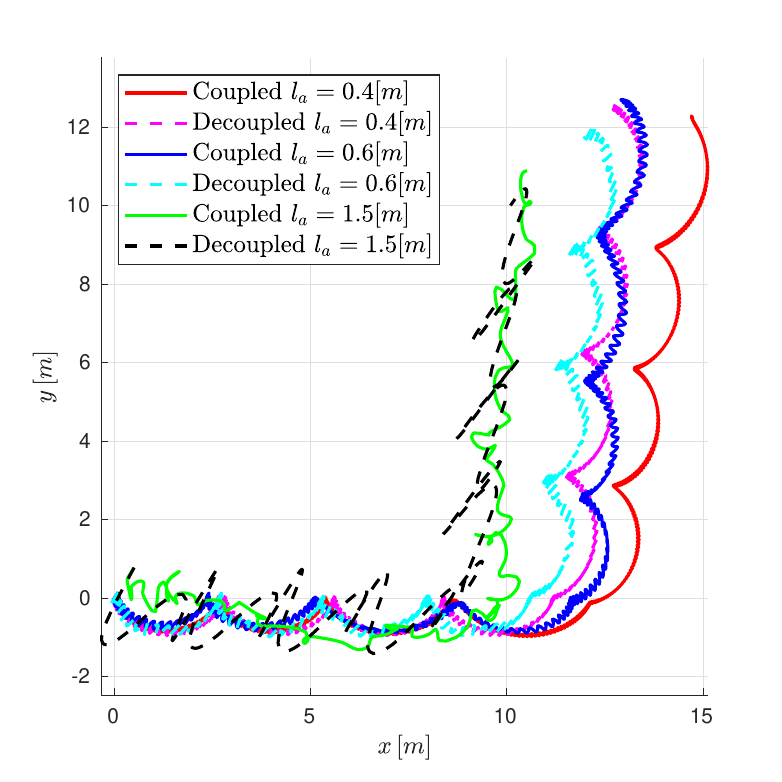}
        \caption{Different arm lengths}
        \label{fig:arm_length1}
    \end{subfigure}
    \hfill
    \begin{subfigure}[t]{0.48\linewidth}
        \centering
        \includegraphics[scale=0.28]{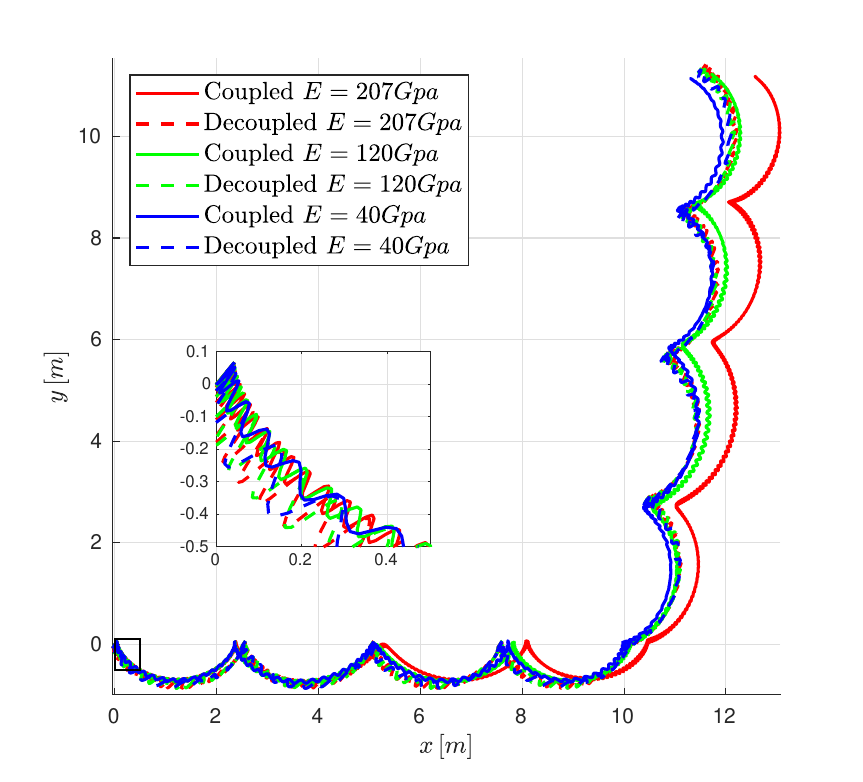}
        \caption{Different arm stiffness}
        \label{fig:stiffness}
    \end{subfigure}

    \vspace{0.5em}
    
\begin{subfigure}[t]{0.48\linewidth}
    \centering
    \raisebox{1.2mm}{%
        \includegraphics[scale=0.27]{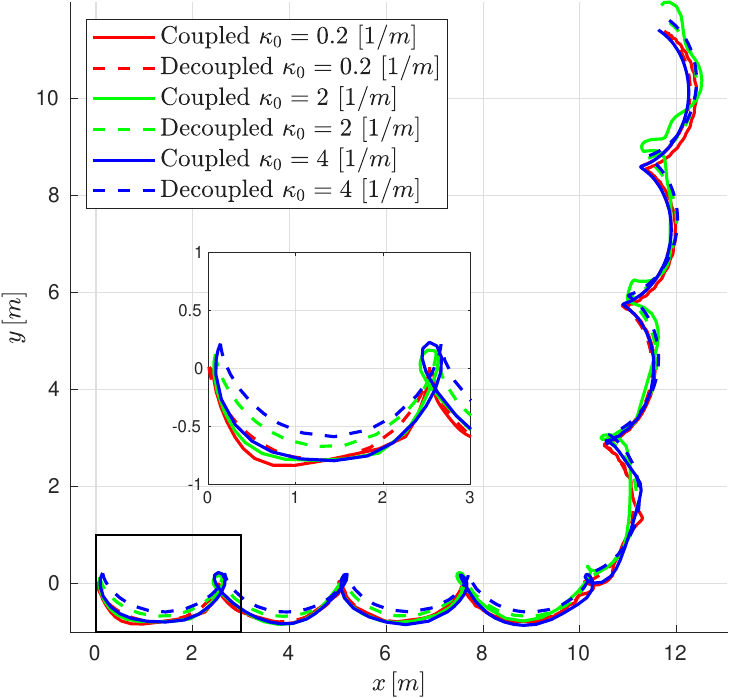}
    }
    \caption{Different arm bending}
    \label{fig:arm_bending}
\end{subfigure}
    \hfill
    \begin{subfigure}[t]{0.48\linewidth}
        \centering
        \includegraphics[scale=0.28]{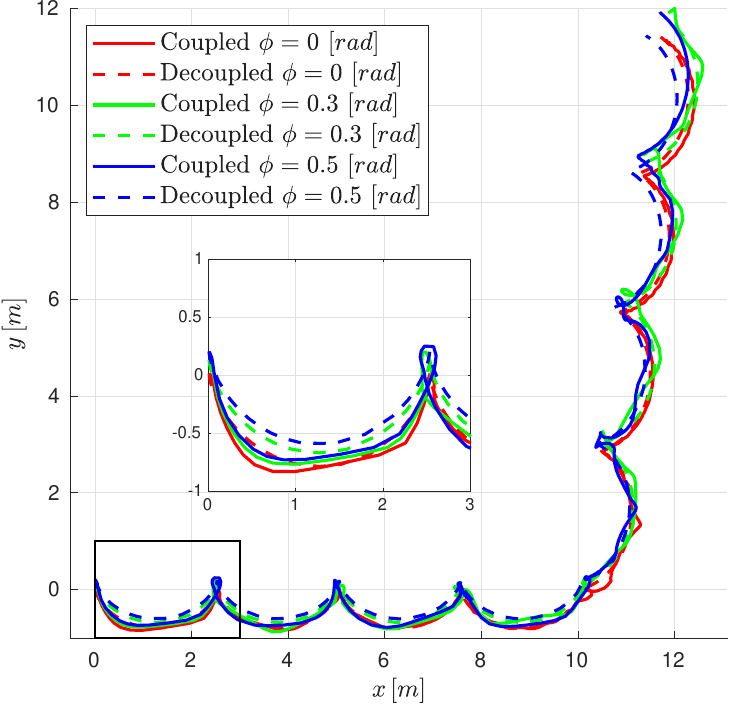}
        \caption{Different UAV roll angles}
        \label{fig:roll_ang}
    \end{subfigure}

    \caption{Open-loop responses for various system parameters.}
    \label{fig:open_loop_params}
\end{figure}

\subsection{Closed-loop Visual Tracking Analysis}

In the closed-loop evaluation, both the coupled and decoupled dynamic formulations are integrated into the DPD-SM-IBVS control architecture and assessed under a time-varying target trajectory $\bm{v}_t$. It is important to emphasize that both approaches employ the same control law; therefore, any observed performance differences between the coupled and decoupled cases arise solely from differences in the underlying dynamic model. These modeling differences directly affect the evolution of the system states and the Jacobian updates within the control loop.

Fig.~\ref{fig:MRAL} depicts the end-effector and reference trajectories projected onto the XY plane of the inertial frame over a $50$-second simulation horizon. The nonstationary world points are defined such that their geometric centroid traces the letters of the word ``MRAL'' across a set of prescribed trajectories. This word is selected because it contains straight, inclined, and curved segments, thereby providing a diverse set of motion primitives. The world points lie on a plane parallel to the inertial-frame XY plane. The figure highlights the end-effector motion under both modeling representations and the resulting deviations between them. To further interpret the observed discrepancies, the norm of the system velocity profile is shown in Fig.~\ref{fig:Mspeed} for the letter M. This allows the performance differences to be analyzed not only with respect to the geometric characteristics of the reference path, but also as a function of the platform's instantaneous velocity.

\begin{figure*}[h!]
    \centering
    \includegraphics[scale = 0.27]{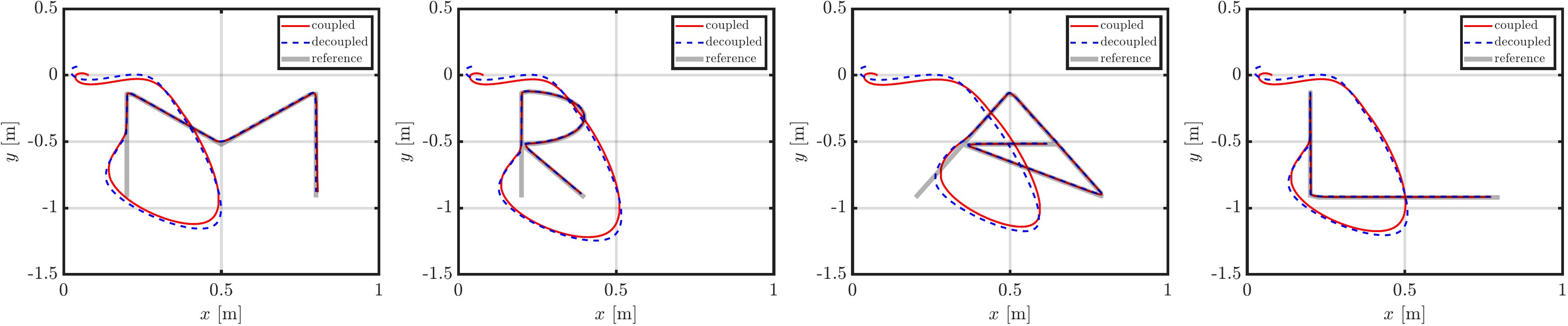} 
    \caption{The xy-plane tip trace of the ACM for coupled and decoupled under the DPD-SM-IBVS law.}
    \label{fig:MRAL}
\end{figure*}

To quantify the corresponding disparity in the image space, a norm-difference metric is defined as
$\text{DS} = \|\bm{e}_s\|_{c} - \|\bm{e}_s\|_{d}$,
where $\|\bm{e}_s\|_{c}$ and $\|\bm{e}_s\|_{d}$ denote the image-feature error norms obtained under the coupled and decoupled formulations, respectively. Fig.~\ref{fig:MRALe} illustrates the temporal evolution of this metric for the letter M. 

The comparative results indicate that the most pronounced disparities between the two formulations occur in regions with high curvature and near sharp corner transitions; however, their transient responses along linear segments remain largely similar. Moreover, the DS metric shows that the decoupled formulation deviates only slightly from the coupled formulation, remaining below $0.7~\mathrm{pixels}$ across all four test scenarios. This result suggests a negligible impact on visual tracking performance while reducing computational burden. Using the Simulink Profiler in MATLAB, the coupled model required $32~\mathrm{ms}$ per sampling instant, whereas the decoupled model required $22~\mathrm{ms}$, underscoring the computational advantage of the decoupled representation. Finally, the results also indicate a dependence on platform speed. Overall, when the platform velocity is high and the vision-based controller has not yet converged, larger discrepancies can occur because the coupling terms become less negligible.

\begin{figure}[h!]
    \centering
    \includegraphics[scale = 0.27, trim={0cm 0cm 0cm 0cm}, clip]{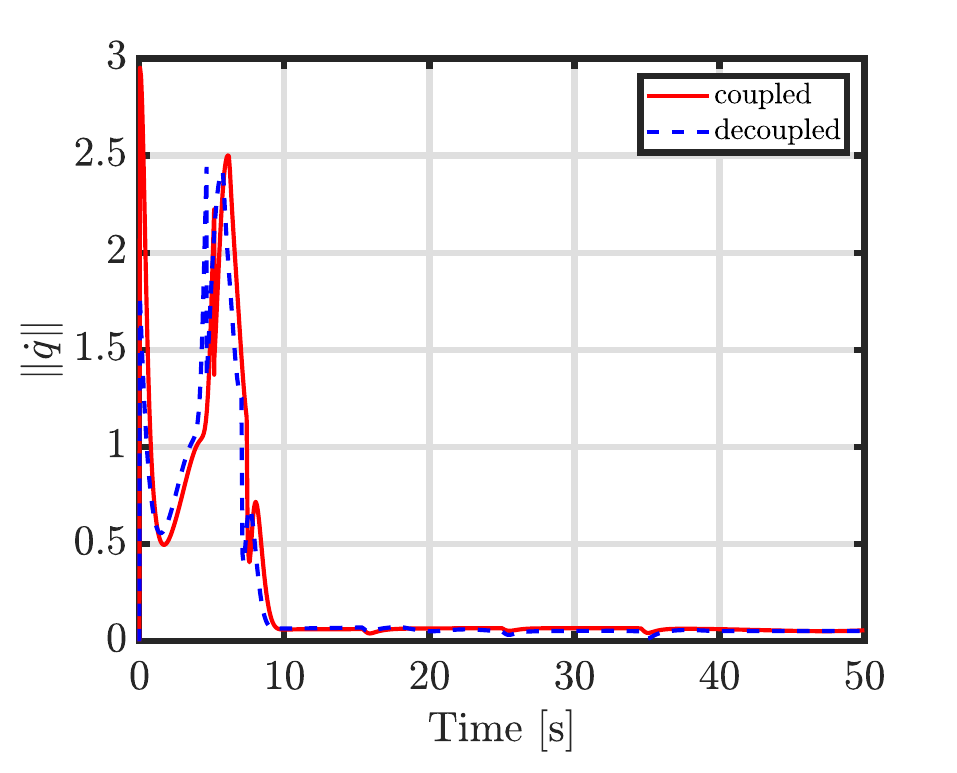} 
    \caption{The joint-space velocity norm for the letter \textit{M}.}
    \label{fig:Mspeed}
\end{figure}

\begin{figure}[h!]
    \centering
    \includegraphics[width = 0.5\columnwidth, height = 3.5cm, trim={0.9cm 0cm 0cm 0cm}, clip]{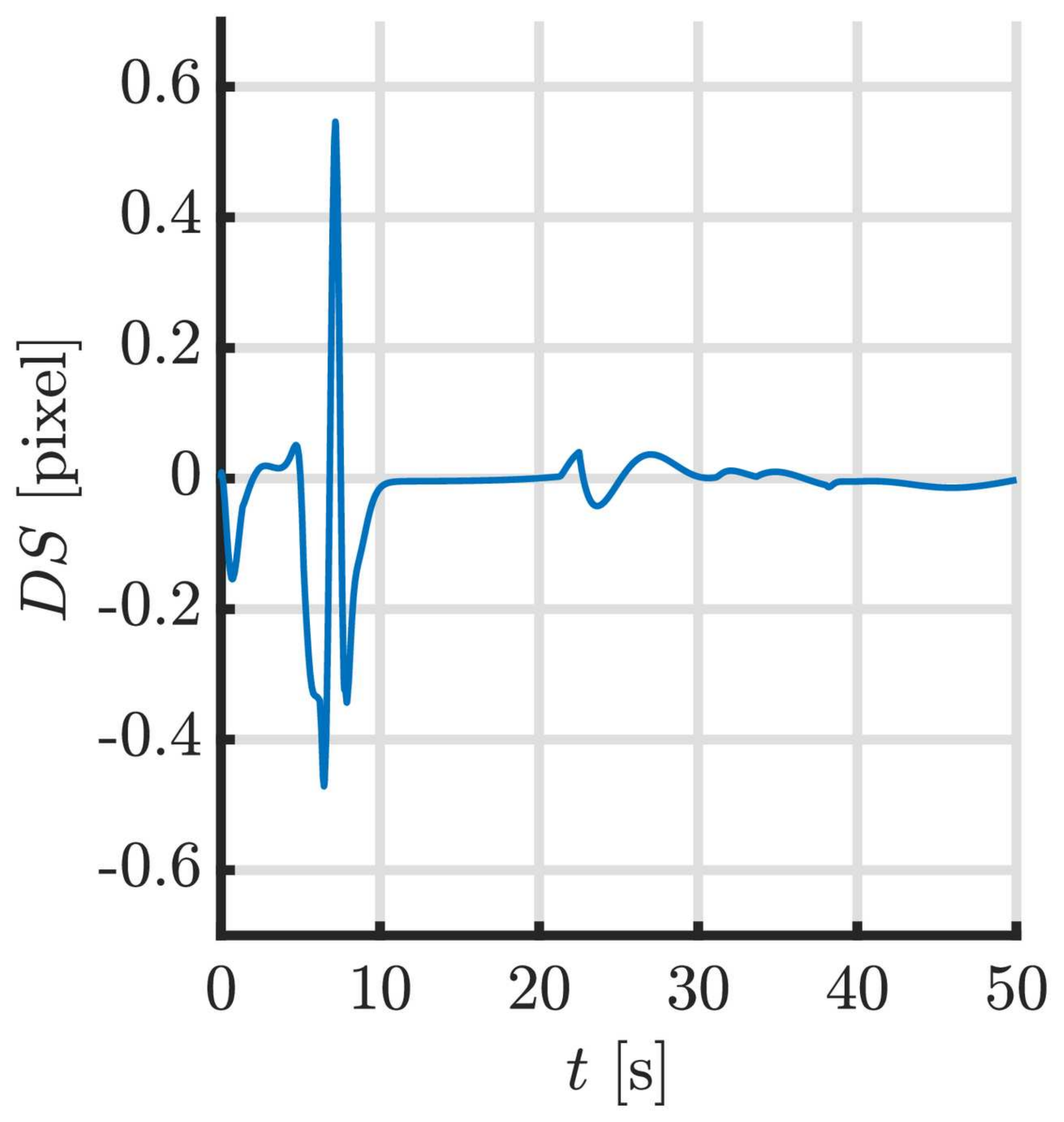} 
    \caption{The norm difference of the image feature error for the letter \textit{M}.}
    \label{fig:MRALe}
\end{figure}

\section{Conclusions}
While ACMs offer distinct advantages for precise tip positioning, navigation in confined environments, and natural compliance during aerial physical interactions, they still suffer from the high computational cost of coupled dynamic modeling which mainly arises from the nonlinearity of the embedded elastic elements. Such models capture the full states of both the aerial platform and the CR, accounting for their mutual interactions. However, this comprehensive formulation significantly increases model complexity and reduces computational efficiency. The analysis presented in this study shows that, under certain conditions, a decoupled model can provide accuracy comparable to the coupled dynamics, while achieving improved computational efficiency. This is made possible by eliminating selected coupling terms and neglecting the inertia interactions between the subsystems.

In open-loop studies, the structural configuration of the ACM plays a critical role in the contribution of coupling terms. This work focuses on a conventional prototype consisting of a single-section CR mounted downward at the center of the UAV. The results indicate that certain parameters strongly influence the validity of the modeling approach. For instance, in the case of long CRs, the decoupled model fails to accurately capture the oscillatory motion of the system. Conversely, when the CR is passive and merely carried as a payload by the UAV prior to manipulation, the decoupled model exhibits minimal discrepancy from the coupled model, regardless of the CR stiffness. Another noteworthy observation concerns the discrepancies between models under different actuation modes. While UAV actuation primarily contributes to rotational discrepancies in the ACM models, CR actuation, even when significantly smaller in magnitude than UAV actuation, has a pronounced effect on positional errors between models.

The results of the closed-loop control of the ACM indicate that both decoupled and coupled models can be successfully implemented to accomplish vision-guided tasks. The discrepancy between the models is primarily observed during the initial stages of trajectory tracking, while in the final stages both achieve similar accuracy. This suggests that employing a robust controller can compensate for modeling inaccuracies, allowing a simplified model to remain valid for successful task execution.

\appendices
\section{The geometric Jacobian of the CR}
\tiny
\begin{equation*}
\begin{aligned}
&\bm{J}_p=\\& 
\begin{bmatrix}
    \dfrac{\cos(\psi_a)\big(\cos(\kappa s) - 1\big)}{\kappa^2} 
    + \dfrac{s \sin(\kappa s) \cos(\psi_a)}{\kappa} 
    & \dfrac{\sin(\psi_a)\big(\cos(\kappa s) - 1\big)}{\kappa} \\

    \dfrac{\sin(\psi_a)\big(\cos(\kappa s) - 1\big)}{\kappa^2} 
    + \dfrac{s \sin(\kappa s) \sin(\psi_a)}{\kappa} 
    & -\dfrac{\cos(\psi_a)\big(\cos(\kappa s) - 1\big)}{\kappa} \\

    \dfrac{\sin(\kappa s)}{\kappa^2} 
    - \dfrac{s \cos(\kappa s)}{\kappa} 
    & 0
\end{bmatrix}.
\end{aligned}
\label{eq:Jp}
\end{equation*}
\normalsize

\section{The off-diagonal elements of the dynamics matrices}


\begin{equation*}
\tiny
\begin{split}
 & M_{17}=   \frac{r_a^{2} \rho \pi}{\kappa^{2}} \bigg[ \left(l \sin\left(\frac{\kappa s}{l}\right) - \kappa s \cos\left(\frac{\kappa s}{l}\right)\right) (\sin(\phi) \sin(\psi) + \cos(\phi) \cos(\psi) \sin(\theta)) -\\
&  \sin(\psi_a) (\cos(\phi) \sin(\psi) - \cos(\psi) \sin(\phi) \sin(\theta)) \left(l \cos\left(\frac{\kappa s}{l}\right) - l + \kappa s \sin\left(\frac{\kappa s}{l}\right)\right) +\\
 & \cos(\psi_a) \cos(\psi) \cos(\theta) \left(l \cos\left(\frac{\kappa s}{l}\right) - l + \kappa s \sin\left(\frac{\kappa s}{l}\right)\right) \bigg],
\end{split}
\end{equation*}

\begin{equation*}
\tiny
\begin{split}
 M_{18} = & \frac{l r_a^{2} \rho \pi}{\kappa} \left(\cos\left(\frac{\kappa s}{l}\right) - 1\right)  \\
 & \left(\cos(\phi) \cos(\psi_a) \sin(\psi) + \cos(\psi) \cos(\theta) \sin(\psi_a) - \cos(\psi_a) \cos(\psi) \sin(\phi) \sin(\theta)\right).
\end{split}
\end{equation*}


\bibliography{refs}
\bibliographystyle{ieeetr}
\end{document}